\documentclass{article}

\usepackage[preprint]{neurips_2026}

\usepackage[utf8]{inputenc} %
\usepackage[T1]{fontenc}    %
\usepackage{hyperref}       %
\usepackage{url}            %
\usepackage{booktabs}       %
\usepackage{amsfonts}       %
\usepackage{nicefrac}       %
\usepackage{microtype}      %
\usepackage{xcolor}         %
\usepackage{amsmath,amssymb,amsthm,mathtools}
\usepackage{enumitem}
\usepackage{array}
\usepackage{booktabs}
\usepackage{hyperref}
\usepackage{tikz}
\usetikzlibrary{arrows.meta,positioning,calc,fit,matrix}
\title{Embedding Dimension Lower Bounds for Universality of Deep Sets and Janossy Pooling}

\usepackage{amsmath,amssymb,amsthm}
\usepackage{hyperref}
\usepackage{enumitem}

\newtheorem{theorem}{Theorem}
\newtheorem{proposition}{Proposition}
\newtheorem{corollary}{Corollary}
\newtheorem{lemma}{Lemma}

\newtheorem{definition}{Definition}

\newcommand{\Csym}{C_{\mathrm{sym}}}

\newcommand{\muniv}{M_{\mathrm{univ}}}

\author{%
  Ali Syed \\
  Department of Mathematics\\
  Texas A\&M University\\
  College Station, TX 77843 \\
  \texttt{saabidi1@tamu.edu} \\
  \And
  Aditya Nambiar \\
  Department of Mathematics\\
  Texas A\&M University\\
  College Station, TX 77843 \\
  \texttt{aditnambiar@tamu.edu} \\
  \And
  Jonathan W. Siegel\\
  Department of Mathematics\\
  Texas A\&M University\\
  College Station, TX 77843 \\
  \texttt{jwsiegel@tamu.edu} \\
}

\begin{document}

\maketitle

\begin{abstract}
  In many practical applications it is important to build symmetries into neural network architectures. Consider the important case of permutation symmetry on point clouds consisting of $n$ points in $d$ dimensions. In this case the network learns a function on a set of $n$ points in $\mathbb{R}^d$, and a natural paradigm for constructing invariant networks is Janossy pooling, which generalizes the popular Deep Sets architecture. We study the universality of this approach, in particular the important question of how large the embedding dimension must be to guarantee universality of this architecture. Specifically, using a novel technique, we prove new lower bounds on the required size of this embedding dimension. For Deep Sets, this gives the correct minimal dimension up to a constant factor for all $d > 1$. For $k$-ary Janossy pooling, we prove the first non-trivial lower bound on the required embedding dimension when $k > 1$.
\end{abstract}

\section{Introduction}
Many applications of machine learning, especially in science and engineering, involve functions which are invariant to known symmetries. One prominent example of this is permutation invariance, which involves learning functions defined on \textit{sets}, as opposed to sequences, of inputs. This motivates the development of neural network architectures which parameterize functions invariant to the action of permutations. 

One of the first such architecture to be introduced is the Deep Sets \cite{zaheer2017deep} model, defined for inputs \(\mathbf{X} = (\mathbf{x}_1,\ldots,\mathbf{x}_n)\in \mathbb{R}^{d\times n}\) by
\begin{equation}
    f(\mathbf{X}) = \rho\left(\sum_{i=1}^n \Phi(\mathbf{x}_i)\right),
\end{equation}
where \(\Phi:\mathbb{R}^d\rightarrow \mathbb{R}^M\) and \(\rho:\mathbb{R}^M\rightarrow \mathbb{R}^p\) are standard multi-layer perceptrons (MLPs). In other words, Deep Sets parameterizes a function on sets by first encoding the elements of the set into an \(M\)-dimensional latent space, aggregating these representations by summing, and finally mapping the resulting latent vector to an output in \(\mathbb{R}^p\). Deep Sets is a popular permutation invariant architecture, which forms the basis for two widely used point set classifiers, PointNet \cite{qi2017pointnet} and PointNet++ \cite{qi2017pointnet++}.

An important generalization of Deep Sets is Janossy pooling \cite{murphy2019janossy}. This paradigm allows for interactions between multiple elements of the set during the encoding step. Specifically, \(k\)-ary Janossy pooling parameterizes functions on \(\mathbf{X} = (\mathbf{x}_1,\ldots,\mathbf{x}_n)\in \mathbb{R}^{d\times n}\) by
\begin{equation}\label{eq:k-ary-janossy}
    f(\mathbf{X}) = \rho\left(\sum_{1\leq i_1,\ldots,i_k\leq n}\Phi(\mathbf{x}_{i_1},\ldots,\mathbf{x}_{i_k})\right),
\end{equation} 
where \(\Phi:\mathbb{R}^{d\times k}\rightarrow \mathbb{R}^M\) and \(\rho:\mathbb{R}^M\rightarrow\mathbb{R}^p\) are MLPs. In other words, \(k\)-ary Janossy pooling maps all ordered \(k\)-tuples of elements in the set into an \(M\)-dimensional latent space before aggregating via a sum and mapping to the output. Deep Sets corresponds to \(k\)-ary Janossy pooling for \(k=1\), while when \(k=2\), Janossy pooling allows pairwise interactions between elements of the input set, and thus bears similarity with the highly successful transformer architecture \cite{vaswani2017attention}.

An important theoretical question regarding these architectures is whether they are universal, i.e., whether any continuous permutation invariant function can be approximated to arbitrary accuracy by Deep Sets or Janossy pooling. General universality for Deep Sets and Janossy pooling has been shown in \cite{zaheer2017deep,wagstaff2022universal,qi2017pointnet,maron2019universality,dym2025low}. 

However, an important subtlety here is the embedding dimension \(M\). For Deep Sets with \(d=1\), i.e., functions on sets of real numbers, it was shown in \cite{wagstaff2022universal} that embedding dimension \(M = n\) is necessary and sufficient for universality. For Deep Sets with \(d > 1\), i.e., functions on sets of points in \(\mathbb{R}^d\), it has been shown in \cite{dym2025low} that an embedding dimension of \(M = 2nd+1\) is sufficient for universality. Obviously, the lower bound \(M \geq n\) proved in \cite{wagstaff2022universal} for \(d = 1\) still holds when \(d > 1\), but no better lower bound is known. For \(k\)-ary Janossy pooling, an embedding dimension of \(M = n\) (if \(d = 1\)) and \(M = 2nd+1\) (if \(d > 1\)) is still sufficient from the corresponding universality for Deep Sets, but no non-trivial lower bound is known.

We remark that the minimal embedding dimension required for universality may depend upon the output dimension $p$. Specifically, it is not difficult to show that the embedding dimension for $k$-ary Janossy pooling must always be at least $M \geq \min\{p,nd\}$, regardless of $k$. This is closely related to the fact that if the embedding maps $\Phi$ are \textit{fixed}, i.e., independent of the target function $f$, then we must have $M \geq nd$. We describe this in more detail in Section \ref{trivial-lower-bounds-section}. Further, all known universality constructions (except in the trivial case $k = n$ where we can take $\Phi = (1/n!)f$ and $\rho$ to be the identity) actually use embedding maps which are independent of the target function $f$. Thus, the key difficulty in the problem lies in the situation where the output dimension $p$ is fixed and the embedding maps may depend upon $f$.

Consequently, in this paper, we focus on the scalar-output case \(p=1\) with \(K=[0,1]^d\), where \(d\ge1\), \(n\ge2\), and \(1\le k\le n\). We write \(C(K^n)\) for the space of continuous real-valued functions on \(K^n\), and \(C_{\mathrm{sym}}(K^n)\) for the subspace of continuous permutation-invariant functions. We now introduce the main latent-dimension quantity studied in the paper.
\begin{definition}
    We denote by $\muniv(d,n,k)$ the minimal embedding dimension required for universality of $k$-ary Janossy pooling acting on sets of $n$ vectors in $\mathbb{R}^d$, with output dimension $p = 1$.
\end{definition}
Determining \(\muniv(d,n,k)\) for different values of \(d\), \(n\), and \(k\) is a significant open problem, highlighted in~\cite{wagstaff2022universal}. Table~\ref{tab:related-bounds} summarizes the best known upper and lower bounds in the literature, together with the new lower bounds proved in this paper.

\begin{table}[t]
\centering
\footnotesize
\setlength{\tabcolsep}{3.0pt}
\renewcommand{\arraystretch}{1.10}
\begin{tabular}{p{2.15cm}p{2.65cm}p{2.65cm}p{3.25cm}}
\toprule
Work & Setting & Upper bound & Lower bound \\
\midrule
Zaheer et al.~\cite{zaheer2017deep}
& \(d=1\), \(k=1\)
& \(\muniv\le n+1\)
& -- \\

Wagstaff et al.~\cite{wagstaff2022universal}
& \(d=1\), \(k=1\)
& \(\muniv\le n\)
& \(\muniv\ge n\) \\

Murphy et al.~\cite{murphy2019janossy}
& \(k=n\)
& \(\muniv = 1\)
& \(\muniv = 1\) \\

Dym--Gortler~\cite{dym2025low}
& \(d\geq 1\), \(k=1\)
& \(\muniv \le 2nd+1\)
& -- \\

\textbf{This paper}
& \(\mathbf{d \geq 1,\ k=1}\)
& --
& \(\mathbf{\muniv\ge d(n-1)}\) \\

\textbf{This paper}
& \(\mathbf{d \geq 1,\ k\geq1}\)
& --
& \(\mathbf{\muniv\gtrsim_k (d(n-k))^{1/k}}\) \\
\bottomrule
\end{tabular}
\vspace{5pt}
\caption{Overview of known upper and lower bounds on the pooled latent dimension for Deep Sets and Janossy pooling. Here \(n\) denotes the set size, \(d\) the ambient dimension, \(k\) the Janossy arity, and \(\muniv\) the latent dimension.}
\label{tab:related-bounds}
\end{table}

\subsection{Our Contributions}

Our main contribution is a new Borsuk-Ulam-based method for proving latent-dimension lower bounds. Applying this method gives the following results.

\begin{enumerate}
    \item For \(k\)-ary Janossy pooling, Corollary~\ref{cor:fixed-k-janossy} shows that for every fixed \(k\ge1\),
    \[
        \muniv(d,n,k)\gtrsim_k (d(n-k))^{1/k}.
    \]
    For \(k>1\), this is the first non-trivial embedding-dimension lower bound for \(k\)-ary Janossy pooling. In particular, the required latent dimension must grow with both \(d\) and \(n\).

    \item For Deep Sets, corresponding to \(k=1\), Corollary~\ref{cor:deepsets} gives
    \[
        \muniv(d,n,1)\ge d(n-1).
    \]
    When $d > 1$, this gives the first lower bound which grows with $d$.
\end{enumerate}
\section{Bounds with fixed $\Phi$ and large $p$}\label{trivial-lower-bounds-section}
In this section, we consider two easier settings: the feature map \(\Phi\) is fixed, or the output dimension \(p\) is large, with \(f:K^n\to\mathbb R^p\). And we denote the pooling latent dimension by \(M\) for this setting. In this situation it is much easier to obtain lower bounds, and we present this primarily to describe the key difficulty in determining $\muniv(d,n,k)$.

The key to this is the following simple observation.
\begin{proposition}[Fixed-feature injectivity obstruction]\label{prop:fixed-feature-injectivity}
    Suppose that the map $\Phi$ in \eqref{eq:k-ary-janossy} is \textit{fixed}, i.e., independent of $f$. Then the representation \eqref{eq:k-ary-janossy} can only be universal (note that we can now only vary $\rho$) if the map
    \begin{equation}\label{eq:fixed-feature-latent-map}
        \Psi(\mathbf{X}) := \sum_{1\leq i_1,\ldots,i_k\leq n}\Phi(\mathbf{x}_{i_1},\ldots,\mathbf{x}_{i_k}).
    \end{equation}
    is injective (modulo permutations).
\end{proposition}
\begin{proof}
Suppose \(\Psi(\mathbf X)=\Psi(\mathbf Y)\) for some \(\mathbf X,\mathbf Y\in K^n\) not related by a permutation. Let \(q:K^n\to K^n/\mathfrak S_n\) be the quotient map. Then \(q(\mathbf X)\neq q(\mathbf Y)\). Since \(K^n/\mathfrak S_n\) is compact Hausdorff, there exists \(\bar g\in C(K^n/\mathfrak S_n)\) such that \(\bar g(q(\mathbf X))\neq \bar g(q(\mathbf Y))\). Hence \(g=\bar g\circ q\) belongs to \(\Csym(K^n)\) and satisfies \(g(\mathbf X)\neq g(\mathbf Y)\).
However, every function of the form \(\rho\circ\Psi\) satisfies \(\rho(\Psi(\mathbf X))=\rho(\Psi(\mathbf Y))\). Therefore \(g\) cannot be uniformly approximated by such functions, contradicting universality. Hence \(\Psi\) must be injective modulo permutations.
\end{proof}
The following proposition records two simpler settings in which lower bounds on \(M\) follow from standard embedding obstructions.
\begin{proposition}[Elementary embedding obstructions]
\label{prop:elementary-embedding-obstruction}
The following elementary lower bounds hold.

\begin{enumerate}[label=(\roman*)]
    \item If the feature map \(\Phi\) is fixed and only the decoder \(\rho\) may vary, then scalar-valued universality in \(\Csym(K^n)\) forces
    \[
        M\ge nd.
    \]

    \item More generally, for \(\mathbb R^p\)-valued invariant targets, exact representation of every continuous permutation-invariant map \(K^n\to\mathbb R^p\) forces
    \[
        M\ge \min\{p,nd\}.
    \]
\end{enumerate}
\end{proposition}

\begin{proof}
We prove (i). Let \(\Psi\) be the fixed latent map defined in \eqref{eq:fixed-feature-latent-map}. Since \(\Phi\) is fixed, every representable scalar-valued function has the form \(\rho\circ\Psi\). By Proposition~\ref{prop:fixed-feature-injectivity}, scalar-valued universality forces \(\Psi\) to be injective modulo permutations.
Because \(\Psi\) is permutation invariant, it factors through the quotient map \(q:K^n\to K^n/\mathfrak S_n\). Thus there exists a continuous map \(\overline{\Psi}:K^n/\mathfrak S_n\to\mathbb R^M\) such that \(\Psi=\overline{\Psi}\circ q\). Injectivity modulo permutations means exactly that \(\overline{\Psi}\) is injective.
The quotient \(K^n/\mathfrak S_n\) has topological dimension \(nd\). Hence, by invariance of domain, it cannot admit a continuous injective map into \(\mathbb R^M\) when \(M<nd\). Therefore \(M\ge nd\).

The proof of (ii) is deferred to Appendix.
\end{proof}
Proposition~\ref{prop:elementary-embedding-obstruction} explains why the fixed-feature and large-output settings are easier: in both cases, the problem reduces to an embedding obstruction. The scalar-output, target-dependent setting defining \(\muniv(d,n,k)\) is different. A scalar invariant target \(f:K^n\to\mathbb R\) need not embed the quotient \(K^n/\mathfrak S_n\), since the encoder \(\Phi\) may vary with \(f\). For example, when $k = n$ we have universality with $M = 1$ by taking $\Phi = f$. Thus, there is no single latent map whose injectivity can be tested. Instead, we must show that every admissible encoder below the claimed dimension threshold has a collision between two separated compact sets. This is exactly the role of the Borsuk-Ulam argument in Theorem~\ref{thm:indexed-janossy-main}.

\section{Reduction to indexed Janossy pooling}\label{subsec:indexed-reduction}

In this section, we show that universality for $k$-ary Janossy pooling implies universality for a non-invariant architecture that we call indexed Janossy pooling. Throughout this section, we let $K = [0,1]^d$. 
\begin{definition}[Indexed $k$-Janossy map] \label{def:indexed-janossy}
    An indexed $k$-Janossy function with latent dimension $M$ on $K^{n}$ is of the form
    \begin{equation}\label{eq:k-ary-janossy-indexed}
        f(\mathbf{X}) = \rho\left(
    \sum_{1\le i_1,\ldots,i_k\le n}
    \phi_{i_1,\ldots,i_k}(\mathbf{x}_{i_1},\ldots,\mathbf{x}_{i_k})
    \right),
    \end{equation}
    where each \(\phi_{i_1,\ldots,i_k}:K^k\to\mathbb R^M\) and \(\rho:\mathbb R^M\to\mathbb R\) are continuous. We define \(M_{\mathrm{ind}}(d,n,k)\) to be the smallest \(M\) for which functions of the form \eqref{eq:k-ary-janossy-indexed} are dense in \(C(K^n)\).
\end{definition}
We will also introduce a definition which captures the feature map in an indexed $k$-Janossy function.
\begin{definition}[Indexed \(k\)-Janossy latent maps]
A continuous map \(\Phi:K^n\to\mathbb R^M\) is called an indexed \(k\)-Janossy latent map if
\begin{equation}\label{eq:indexed-janossy}
    \Phi(\mathbf{x}_1,\ldots,\mathbf{x}_n)
    =
    \sum_{1\le i_1,\ldots,i_k\le n}
    \phi_{i_1,\ldots,i_k}(\mathbf{x}_{i_1},\ldots,\mathbf{x}_{i_k}),
\end{equation}
where each \(\phi_{i_1,\ldots,i_k}:K^k\to\mathbb R^M\) is continuous.
\end{definition}

The indexed $k$-Janossy maps are a larger class of functions than the original \(k\)-ary Janossy class \eqref{eq:k-ary-janossy}, because the encoder is allowed to depend on the ordered tuple \((i_1,\ldots,i_k)\). The usual \(k\)-ary Janossy architecture in \eqref{eq:k-ary-janossy} is the special case where
\[
    \phi_{i_1,\ldots,i_k}=\phi
    \qquad
    \text{for every }(i_1,\ldots,i_k).
\]

\subsection{Relationship with the Kolmogorov-Arnold superposition theorem} The Kolmogorov-Arnold superposition theorem \cite{kolmogorov1957representations,kahane1975theoreme} states that any continuous function $f:[0,1]^n$ can be expressed as
\begin{equation}
    f(\mathbf{x}) = \sum_{q=1}^{2n+1}
\Psi_q\!\left(
\sum_{p=1}^n \psi_{pq}(x_p)
\right),
\end{equation}
where $\Psi_q$ and $\psi_{pq}$ for $q=1,...,2n+1$ and $p = 1,...,n$ are continuous functions. We observe that this is equivalent to indexed $1$-Janossy pooling with the added restriction that the outer function $\rho$ is of the form
\begin{equation}
    \rho(\mathbf{y}) = \sum_{q=1}^{2n+1} \Psi_q(y_q),
\end{equation}
i.e., that $\rho$ is additively separable. With this additional restriction, it is known that a latent dimension of $2n+1$ is necessary \cite{sternfeld1985dimension,levin1990dimension}. We remark that the Kolmogorov-Arnold theorem has recently been successfully applied to deep learning in the form of Kolmogorov Arnold networks (KAN) \cite{liu2024kan}.

\subsection{Properties of indexed $k$-Janossy latent maps}\label{sec:k-janossy-properties}
In this subsection, we analyze properties of indexed $k$-Janossy latent maps which are critical for the proof of Theorem \ref{thm:indexed-janossy-main}. We have the following lemma.

\begin{lemma}[Finite-difference rigidity]\label{lem:janossy-finite-difference}
Let \(\Phi:K^n\to\mathbb R^M\) be an indexed \(k\)-Janossy latent map of the form \eqref{eq:indexed-janossy}. Let \(\mathbf a_i,\mathbf y_i\in\mathbb R^d\) for \(i=1,\ldots,k\), and suppose
\[
    \mathbf{a}_i+\eta_i \mathbf{y}_i\in K
    \qquad
    \text{for all }i=1,\ldots,k,\ \eta_i\in\{0,1\}.
\]
Then
\begin{equation}\label{eq:janossy-finite-difference}
    \sum_{\eta\in\{0,1\}^k}
    (-1)^{|\eta|}
    \Phi(\mathbf{a}_1+\eta_1\mathbf{y}_1,\ldots,\mathbf{a}_k+\eta_k\mathbf{y}_k,\mathbf{x}_{k+1},\ldots,\mathbf{x}_n)
\end{equation}
is independent of \(\mathbf{x}_{k+1},\ldots,\mathbf{x}_n\).
\end{lemma}

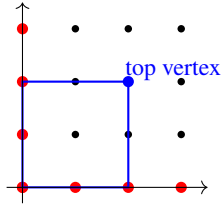
\begin{figure}[t]
\centering
\begin{tikzpicture}[scale=0.70, every node/.style={font=\small}]
\foreach \x in {0,1,2,3}
  \foreach \y in {0,1,2,3}
    \fill (\x,\y) circle (2pt);

\foreach \x in {0,1,2,3}
  \fill[red] (\x,0) circle (3pt);
\foreach \y in {0,1,2,3}
  \fill[red] (0,\y) circle (3pt);

\draw[blue, thick] (0,0) -- (2,0) -- (2,2) -- (0,2) -- cycle;
\fill[blue] (2,2) circle (3pt);

\node[blue] at (2.85,2.2) {top vertex};

\draw[->] (-0.3,0) -- (3.5,0);
\draw[->] (0,-0.3) -- (0,3.5);
\end{tikzpicture}
\caption{For \(k=2\), the black points represent the base grid \(G^2\), while the red points represent the axis subset \(A\). A non-axis grid point \((\mathbf y_1,\mathbf y_2)\) is the top vertex of the rectangle whose other vertices \((\mathbf 0,\mathbf 0)\), \((\mathbf y_1,\mathbf 0)\), and \((\mathbf 0,\mathbf y_2)\) all lie in \(A\).}
\label{fig:axes-grid-k2}
\end{figure}

\begin{proof}
The proof is deferred to Appendix \ref{app:lemma1proof}.
\end{proof}

Next, let \(G\subset K=[0,1]^d\) be a finite set of \(d\)-dimensional vectors such that
\[
\mathbf 0\in G,
\qquad
|G|=\left\lceil(d(n-k)+1)^{1/k}\right\rceil .
\]
Such a set can certainly always be found since $K$ is an infinite set. 
Consider \(G^k\subset K^k\), whose elements are \(k\)-tuples
\((\mathbf y_1,\ldots,\mathbf y_k)\) with each \(\mathbf y_j\in G\), and define its `axis' subset $A$ (which depends upon $G$) via
\[
A=\{(\mathbf{y}_1,\ldots,\mathbf{y}_k)\in G^k:\mathbf{y}_j=\mathbf{0}\text{ for at least one }j\}.
\]
Note that
\begin{equation}\label{eq:axes-cardinality}
|A|
=
\left\lceil(d(n-k)+1)^{1/k}\right\rceil^k
-
\left(\left\lceil(d(n-k)+1)^{1/k}\right\rceil-1\right)^k .
\end{equation}
The case \(k=2\) is illustrated in Figure~\ref{fig:axes-grid-k2}.

Lemma \ref{lem:janossy-finite-difference} implies the following crucial property of indexed $k$-Janossy maps.

\begin{proposition}[Equality from axes to the full grid]\label{prop:axes-to-grid}
Let \(\Phi:K^n\to\mathbb R^M\) be an indexed \(k\)-Janossy latent map of the form \eqref{eq:indexed-janossy}. Suppose that for two tail configurations \(\mathbf{u}_e,\mathbf{v}_e\in K^{n-k}\),
\[
    \Phi(\mathbf{x}_b,\mathbf{u}_e)=\Phi(\mathbf{x}_b,\mathbf{v}_e)
    \qquad
    \text{for every }\mathbf{x}_b\in A.
\]
Then
\[
    \Phi(\mathbf{x}_b,\mathbf{u}_e)=\Phi(\mathbf{x}_b,\mathbf{v}_e)
    \qquad
    \text{for every }\mathbf{x}_b\in G^k.
\]
\end{proposition}

\begin{proof}[Proof of Proposition~\ref{prop:axes-to-grid}]
The conclusion is already assumed on \(A\). Fix \(\mathbf{x}_b=(\mathbf{y}_1,\ldots,\mathbf{y}_k)\in G^k\setminus A\). Then every \(\mathbf{y}_i\neq0\). Consider the cube vertices
\[
(\eta_1\mathbf{y}_1,\ldots,\eta_k\mathbf{y}_k),
\qquad
\eta\in\{0,1\}^k.
\]
The top vertex \(\eta=(1,\ldots,1)\) is \(\mathbf{x}_b\), and every other vertex lies in \(A\).

Apply Lemma~\ref{lem:janossy-finite-difference} with base point \((0,\ldots,0)\) and increments \(\mathbf{y}_1,\ldots,\mathbf{y}_k\). The alternating sum
\[
    \sum_{\eta\in\{0,1\}^k}
    (-1)^{|\eta|}
    \Phi(\eta_1\mathbf{y}_1,\ldots,\eta_k\mathbf{y}_k,\mathbf{z}_e)
\]
is independent of \(\mathbf{z}_e\). Hence its value at \(\mathbf{z}_e=\mathbf{u}_e\) equals its value at \(\mathbf{z}_e=\mathbf{v}_e\). All terms except the top-vertex term agree by the assumption on \(A\), so they cancel, leaving
\[
    \Phi(\mathbf{y}_1,\ldots,\mathbf{y}_k,\mathbf{u}_e)=\Phi(\mathbf{y}_1,\ldots,\mathbf{y}_k,\mathbf{v}_e).
\]
Thus equality holds for every \(\mathbf{x}_b\in G^k\). See again Figure \ref{fig:axes-grid-k2}.
\end{proof}

\subsection{Reduction lemma}

Next, we show that lower bounds on $M_{\mathrm{ind}}(d,n,k)$ imply lower bounds on $\muniv(d,n,k)$. This reduces our problem to proving lower bounds on the universality of indexed $k$-Janossy pooling. 

To see this first in one dimension, choose \(n\) distinct points \(z_1,\ldots,z_n\in(0,1)\), and choose \(\delta>0\) so small that the intervals \(z_i+\delta[0,1]\) are contained in \((0,1)\) and are pairwise disjoint. Then
\[
    \{(z_1+\delta t_1,\ldots,z_n+\delta t_n):(t_1,\ldots,t_n)\in[0,1]^n\}
\]
is a copy of \([0,1]^n\) inside the original domain where the \(i\)-th coordinate is identified by the interval it lies in. Since the intervals are disjoint, different permutations of this copy are disjoint. Hence, given any \(g\in C([0,1]^n)\), one can define a permutation-invariant continuous function \(f\) such that
\[
    f(z_1+\delta t_1,\ldots,z_n+\delta t_n)
    =
    g(t_1,\ldots,t_n)
    \qquad
    \text{for every }(t_1,\ldots,t_n)\in[0,1]^n.
\]

In general dimension, we use the same construction with disjoint cubes. Choose pairwise disjoint closed cubes
\[
    C_1,\ldots,C_n\subset(0,1)^d
\]
and affine homeomorphisms \(T_j:K\to C_j\). The set
\[
    L=\{(T_1\mathbf{t}_1,\ldots,T_n\mathbf{t}_n):(\mathbf{t}_1,\ldots,\mathbf{t}_n)\in K^n\}
    \subset K^n
\]
is a copy of \(K^n\) in which the cube \(C_j\) identifies the \(j\)-th input. Therefore, on \(L\), a permutation-invariant function can still agree with an arbitrary ordered function \(g(\mathbf{t}_1,\ldots,\mathbf{t}_n)\), because the order is encoded by the cubes.

The next lemma states the consequence for \(k\)-ary Janossy pooling. If the k-ary Janossy architecture were universal on \(\Csym(K^n)\), then its restriction to \(L\) would approximate every \(g\in C(K^n)\). But after this restriction, the $k$-ary Janossy encoder in \ref{eq:k-ary-janossy} behaves like an indexed encoder: the term indexed by \((i_1,\ldots,i_k)\) always receives inputs from the fixed cubes \(C_{i_1},\ldots,C_{i_k}\).

\begin{lemma}[Reduction to indexed Janossy]\label{lem:labeled-copy-reduction}
If the \(k\)-ary Janossy class with latent dimension \(M\) is dense in \(\Csym(K^n)\), then every \(g\in C(K^n)\) can be uniformly approximated on \(K^n\) by functions of the form
\[
    (\mathbf{t}_1,\ldots,\mathbf{t}_n)
    \longmapsto
    \rho\left(
    \sum_{1\le i_1,\ldots,i_k\le n}
    \phi_{i_1,\ldots,i_k}(\mathbf{t}_{i_1},\ldots,\mathbf{t}_{i_k})
    \right),
\]
where each \(\phi_{i_1,\ldots,i_k}:K^k\to\mathbb R^M\) and \(\rho:\mathbb R^M\to\mathbb R\) are continuous. This implies that
\begin{equation}
    \muniv(d,n,k) \geq M_{\mathrm{ind}}(d,n,k).
\end{equation}
\end{lemma}

\begin{proof}
Fix \(g\in C(K^n)\), and let
\[
    L
    =
    \{(T_1\mathbf{t}_1,\ldots,T_n\mathbf{t}_n):
    (\mathbf{t}_1,\ldots,\mathbf{t}_n)\in K^n\}
    \subset K^n .
\]
On \(L\), define
\[
    f_0(T_1\mathbf{t}_1,\ldots,T_n\mathbf{t}_n)
    =
    g(\mathbf{t}_1,\ldots,\mathbf{t}_n).
\]
Extend this definition to all permutations of \(L\) by requiring permutation invariance. This is well-defined because the cubes \(C_1,\ldots,C_n\) are disjoint: every configuration in the union of the permuted copies has a unique point in each cube \(C_j\), so it uniquely determines the ordered tuple \((\mathbf{t}_1,\ldots,\mathbf{t}_n)\). Since the permuted copies are disjoint compact sets, \(f_0\) is continuous on their union.

By Tietze extension, extend \(f_0\) to some \(h\in C(K^n)\), and symmetrize:
\[
    f(\mathbf{x}_1,\ldots,\mathbf{x}_n)
    =
    \frac1{n!}
    \sum_{\sigma\in\mathfrak S_n}
    h(\mathbf{x}_{\sigma(1)},\ldots,\mathbf{x}_{\sigma(n)}).
\]
Then \(f\in\Csym(K^n)\), and on the original labeled copy \(L\),
\[
    f(T_1\mathbf{t}_1,\ldots,T_n\mathbf{t}_n)
    =
    g(\mathbf{t}_1,\ldots,\mathbf{t}_n),
\]
because on each permuted copy the definition of \(f_0\) was chosen to be permutation invariant and therefore agrees with \(g(\mathbf t_1,\ldots,\mathbf t_n)\).

Assume the \(k\)-ary Janossy architecture is dense in \(\Csym(K^n)\). Then \(f\) can be uniformly approximated by functions \(\rho\circ\Phi\), where
\[
    \Phi(\mathbf{x}_1,\ldots,\mathbf{x}_n)
    =
    \sum_{1\le i_1,\ldots,i_k\le n}
    \phi(\mathbf{x}_{i_1},\ldots,\mathbf{x}_{i_k}).
\]
Restricting to \(L\), so that \(\mathbf{x}_i=T_i\mathbf{t}_i\), gives
\[
    \sum_{1\le i_1,\ldots,i_k\le n}
    \phi(T_{i_1}\mathbf{t}_{i_1},\ldots,T_{i_k}\mathbf{t}_{i_k}).
\]
For each ordered tuple \(I=(i_1,\ldots,i_k)\), define
\[
    \phi_I(\mathbf{z}_1,\ldots,\mathbf{z}_k)
    =
    \phi(T_{i_1}\mathbf{z}_1,\ldots,T_{i_k}\mathbf{z}_k).
\]
Thus, on \(L\), the shared latent map becomes the indexed latent map
\[
    \sum_{1\le i_1,\ldots,i_k\le n}
    \phi_{i_1,\ldots,i_k}(\mathbf{t}_{i_1},\ldots,\mathbf{t}_{i_k}).
\]
Since \(f=g\) on \(L\), this indexed architecture approximates \(g\). Hence every \(g\in C(K^n)\) can be approximated by the indexed class.
\end{proof}
Lemma~\ref{lem:labeled-copy-reduction} shows that any universality result for the original \(k\)-ary Janossy architecture would imply universality for the larger indexed class. Therefore, it suffices to prove a lower bound for the larger indexed class.

\section{Main Results}
In this section, we prove the main latent-dimension lower bounds. Using the reduction from the previous section, it suffices to prove a lower bound for the larger indexed Janossy class introduced in Definition~\ref{def:indexed-janossy}. We state this indexed lower bound in Theorem~\ref{thm:indexed-janossy-main} and prove it in Subsection~\ref{subsec:indexed-lower-bound}. The corresponding lower bounds for \(k\)-ary Janossy pooling and Deep Sets are then obtained in Corollaries~\ref{cor:shared-janossy-main} and~\ref{cor:deepsets} respectively.

\subsection{Statement of the lower bounds}

We now state the main lower bound results for the indexed latent dimension \(M_{\mathrm{ind}}(d,n,k)\) and for the $k$-ary Janossy pooling latent dimension \(M_{\mathrm{univ}}(d,n,k)\).
\begin{theorem}[Indexed lower bound]\label{thm:indexed-janossy-main}
Let \(d\ge1\) and \(1\le k<n\). Then
\begin{equation}\label{eq:indexed-main-bound}
M_{\mathrm{ind}}(d,n,k)
\ge
\left\lceil
\frac{d(n-k)}
{
\left\lceil(d(n-k)+1)^{1/k}\right\rceil^k
-
\left(\left\lceil(d(n-k)+1)^{1/k}\right\rceil-1\right)^k
}
\right\rceil .
\end{equation}
\end{theorem}

The $k$-ary Janossy pooling lower bound follows immediately from the indexed one.

\begin{corollary}[k-ary Janossy lower bound]\label{cor:shared-janossy-main}
Let \(d\ge1\) and \(1\le k<n\). Then
\[
\muniv(d,n,k)
\ge
\left\lceil
\frac{d(n-k)}
{
\left\lceil(d(n-k)+1)^{1/k}\right\rceil^k
-
\left(\left\lceil(d(n-k)+1)^{1/k}\right\rceil-1\right)^k
}
\right\rceil .
\]
\end{corollary}

\begin{proof}
If the \(k\)-ary Janossy architecture were universal with latent dimension \(M\), then Lemma~\ref{lem:labeled-copy-reduction} would make the indexed class with the same latent dimension \(M\) dense in \(C(K^n)\). By definition of \(M_{\mathrm{ind}}(d,n,k)\), this forces
\[
    M\ge M_{\mathrm{ind}}(d,n,k).
\]
Thus \(\muniv(d,n,k)\ge M_{\mathrm{ind}}(d,n,k)\), and the result follows from Theorem~\ref{thm:indexed-janossy-main}.
\end{proof}

\begin{corollary}[Fixed \(k\)]\label{cor:fixed-k-janossy}
For every fixed \(k\ge1\),
\[
    \muniv(d,n,k)\gtrsim_k (d(n-k))^{1/k}.
\]
\end{corollary}

\begin{proof}
Let \(s=\left\lceil(d(n-k)+1)^{1/k}\right\rceil\). Since
\[
s^k-(s-1)^k\le k s^{k-1}\le C_k(d(n-k))^{(k-1)/k},
\]
Corollary~\ref{cor:shared-janossy-main} gives
\[
\muniv(d,n,k)\ge c_k(d(n-k))^{1/k}.
\]
\end{proof}

\begin{corollary}[Deep Sets]\label{cor:deepsets}
For Deep Sets, corresponding to \(k=1\),
\[
    \muniv(d,n,1)\ge d(n-1).
\]
\end{corollary}
\begin{proof}
Setting \(k=1\) in Corollary~\ref{cor:shared-janossy-main} gives the Deep Sets lower bound.
\end{proof}

\subsection{Proof of Theorem~\ref{thm:indexed-janossy-main}}\label{subsec:indexed-lower-bound}
We will prove Theorem~\ref{thm:indexed-janossy-main} by contradiction. Suppose \(M\) satisfies
\[
M
\left[
\left\lceil(d(n-k)+1)^{1/k}\right\rceil^k
-
\left(\left\lceil(d(n-k)+1)^{1/k}\right\rceil-1\right)^k
\right] = M|A|
<
d(n-k),
\]
where $A$ is the axis set constructed in Section \ref{sec:k-janossy-properties}.

We will construct two disjoint compact sets \(E^+,E^-\subset K^n\) such that every indexed \(k\)-Janossy latent map \(\Phi\) of the form \eqref{eq:indexed-janossy} takes the same value at some pair \(\mathbf{x}^+\in E^+\), \(\mathbf{x}^-\in E^-\). Note that this is also the strategy followed in \cite{wagstaff2022universal}.

Assume that such sets have been constructed. Since \(E^+\) and \(E^-\) are compact and disjoint, define:
\[
    g(\mathbf{x})
    =
    \frac{\operatorname{dist}(\mathbf{x},E^-)}
    {\operatorname{dist}(\mathbf{x},E^+)+\operatorname{dist}(\mathbf{x},E^-)}.
\]
Then \(g\in C(K^n)\), \(g=1\) on \(E^+\), and \(g=0\) on \(E^-\). Now fix any indexed latent map \(\Phi\) and any continuous decoder \(\rho:\mathbb R^M\to\mathbb R\). By construction, there exist \(\mathbf{x}^+\in E^+\) and \(\mathbf{x}^-\in E^-\) such that
\begin{equation}\label{eq:collision-property}
\Phi(\mathbf{x}^+)=\Phi(\mathbf{x}^-).
\end{equation}

Hence \(\rho(\Phi(\mathbf{x}^+))=\rho(\Phi(\mathbf{x}^-))\), while \(g(\mathbf{x}^+)=1\) and \(g(\mathbf{x}^-)=0\). Therefore
\[
\begin{aligned}
2\sup_{\mathbf{x}\in K^n}|g(\mathbf{x})-\rho(\Phi(\mathbf{x}))|
&\ge
|g(\mathbf{x}^+)-\rho(\Phi(\mathbf{x}^+))|
+
|g(\mathbf{x}^-)-\rho(\Phi(\mathbf{x}^-))|
&\ge 1.
\end{aligned}
\]
So no decoder \(\rho\circ\Phi\) can approximate \(g\) uniformly to error below \(1/2\). 

It remains only to construct \(E^+\) and \(E^-\) with the stated collision property on the encoder $\Phi$. 
To build \(E^+\) and \(E^-\), we use the last \(n-k\) variables as the tail block and the first \(k\) variables as the base block. We write
\[
    \mathbf{x}_b=(\mathbf{x}_1,\ldots,\mathbf{x}_k)\in K^k,
    \qquad
    \mathbf{x}_e=(\mathbf{x}_{k+1},\ldots,\mathbf{x}_n)\in K^{n-k}.
\]
The tail block has dimension \(d(n-k)\), so we parametrize a small sphere inside it. Let
\[
    S^{d(n-k)-1}
    =
    \{\mathbf{u}=(\mathbf{u}_{k+1},\ldots,\mathbf{u}_n)\in(\mathbb R^d)^{n-k}:\|\mathbf{u}\|_2=1\},
\]
choose
\(
\mathbf{c}=\left(\frac12\mathbf{1}_d,\ldots,\frac12\mathbf{1}_d\right)\in K^{n-k}
\) and \(0<\varepsilon<1/2\), and define
\[
    \delta:S^{d(n-k)-1}\to K^{n-k},
    \qquad
    \delta(\mathbf{u})=\mathbf{c}+\varepsilon \mathbf{u},
    \qquad
    \delta(-\mathbf{u})=\mathbf{c}-\varepsilon \mathbf{u}.
\]
Recall the definitions of the set $G^k$ and $A$ given in Section \ref{sec:k-janossy-properties}. We choose closed sets \(B_1,\ldots,B_{d(n-k)+1}\) covering \(S^{d(n-k)-1}\), with no region containing an antipodal pair, i.e. \(B_r\cap(-B_r)=\varnothing\) for every \(r\). It is well-known that such a covering exists (see Appendix \ref{appendix:regions} for an explicit construction). In fact, an equivalent formulation of the Borsuk-Ulam theorem is that no such covering with fewer sets exists \cite{matouvsek2003using}. 
Now, choose an injection
\begin{equation}\label{eq:injective-map}
    \boldsymbol{\chi}:\{1,\ldots,d(n-k)+1\}\hookrightarrow G^k,
\end{equation}
so that \(\boldsymbol{\chi}(r)\) is the base label assigned to region \(B_r\). This is possible since
\[
    |G^k|
    =
    \left\lceil(d(n-k)+1)^{1/k}\right\rceil^k
    \ge d(n-k)+1.
\]
We then define the two candidate obstruction sets by
\begin{equation}\label{eq:Epm}
    E^+=\bigcup_{r=1}^{d(n-k)+1}\{(\boldsymbol{\chi}(r),\delta(\mathbf{u})):\mathbf{u}\in B_r\},
    \qquad
    E^-=\bigcup_{r=1}^{d(n-k)+1}\{(\boldsymbol{\chi}(r),\delta(-\mathbf{u})):\mathbf{u}\in B_r\}.
\end{equation}
These sets are compact, since they are finite unions of compact sets. They are also disjoint. Indeed, if a point lies in both \(E^+\) and \(E^-\), then for some \(r,q\) and \(\mathbf{u}\in B_r\), \(\mathbf{v}\in B_q\), 
\[
    (\boldsymbol{\chi}(r),\delta(\mathbf{u}))
    =
    (\boldsymbol{\chi}(q),\delta(-\mathbf{v})).
\]
Equality of the base blocks gives \(\boldsymbol{\chi}(r)=\boldsymbol{\chi}(q)\), hence \(r=q\) by injectivity. Equality of the tail blocks gives \(\delta(\mathbf{u})=\delta(-\mathbf{v})\), hence \(\mathbf{u}=-\mathbf{v}\). Thus \(\mathbf{u}\in B_r\) and \(-\mathbf{u}=\mathbf{v}\in B_r\), contradicting \(B_r\cap(-B_r)=\varnothing\). 

Now, for a fixed indexed latent map \(\Phi\), apply the Borsuk-Ulam theorem to the map
\begin{equation}\label{eq:borsuk-ulam-map}
    F:S^{d(n-k)-1}\to\mathbb R^{M|A|},
    \qquad
    F(\mathbf{u})=\bigl(\Phi(\mathbf{x}_b,\delta(\mathbf{u}))\bigr)_{\mathbf{x}_b\in A}.
\end{equation}
Since by assumption \(M|A|<d(n-k)\) we obtain antipodal points on the sphere such that \(F(\mathbf{u})=F(-\mathbf{u})\), equivalently
\[
    \Phi(\mathbf{x}_b,\delta(\mathbf{u}))=\Phi(\mathbf{x}_b,\delta(-\mathbf{u}))
    \qquad
    \text{for every }\mathbf{x}_b\in A.
\]
Since $\Phi$ is assumed to be indexed $k$-Janossy, we can now apply Proposition \ref{prop:axes-to-grid} to get
\[
    \Phi(\mathbf{x}_b,\delta(\mathbf{u}))=\Phi(\mathbf{x}_b,\delta(-\mathbf{u}))
    \qquad
    \text{for every }\mathbf{x}_b\in G^k.
\]

In particular, letting $r$ be an index such that $\textbf{u}\in B_r$, we have
\begin{equation}
    \Phi(\boldsymbol{\chi}(r),\delta(\mathbf{u}))=\Phi(\boldsymbol{\chi}(r),\delta(-\mathbf{u})).
\end{equation}
Now $\textbf{x}^+ := (\boldsymbol{\chi}(r),\delta(\mathbf{u}))\in E^+$ and $\textbf{x}^- := (\boldsymbol{\chi}(r),\delta(-\mathbf{u}))\in E^-$, and the previous equation gives \(\Phi(\mathbf{x}^+) = \Phi(\mathbf{x}^-)\). This completes the proof of Theorem \ref{thm:indexed-janossy-main}.

\section{Conclusion}
In this paper, we introduced a Borsuk-Ulam-based method for proving latent-dimension lower bounds for \(k\)-ary Janossy pooled permutation-invariant architectures. For Deep Sets, the method gives \(\muniv(d,n,1)\ge d(n-1)\). Together with the \(2nd+1\) upper bound implied by Dym--Gortler~\cite{dym2025low}, this shows that the universal latent dimension for Deep Sets on \(K=[0,1]^d\) has the correct \(nd\) scaling up to constants. When \(d=1\), our bound gives \(n-1\), recovering the known linear dependence on \(n\) from \cite{wagstaff2022universal} up to an additive constant of $1$.

For fixed \(k>1\), the same method gives \(\muniv(d,n,k)\gtrsim_k (d(n-k))^{1/k}\). This is the first non-trivial dimension-dependent lower bound for \(k\)-ary Janossy pooling, and shows that finite-order Janossy interactions still require a latent dimension which grows with both the ambient dimension and the number of input points.

Although we state the results for \(K=[0,1]^d\), the argument only uses the existence of a small \(d\)-dimensional cube inside the domain. Hence the same rates hold for any compact \(K\subset\mathbb R^d\) with nonempty interior.

The main limitation of the argument in Theorem \ref{thm:indexed-janossy-main} is that the obstruction is not sharp in all regimes. For Deep Sets, the method loses one input point because one block is used to separate the two compact obstruction sets; this explains the \(d(n-1)\) bound rather than a possible \(dn\) bound. For \(d=1\), this appears as the gap between \(n-1\) and the known sharp bound \(n\).

For \(k>1\), the exponent \(1/k\) is a limitation of the current construction. The proof uses \(k\) base variables to label the \(d(n-k)+1\) antipodal-free regions of the tail sphere; since a grid in the base block has size \(|G|^k\), the natural choice is \(|G|\approx (d(n-k))^{1/k}\). This leaves a substantial gap to the inherited \(O(nd)\) upper bound from Deep Sets. Determining the correct dependence on \(d\), \(n\), and \(k\), and proving matching bounds for finite-order Janossy pooling, remains an interesting open problem.
\bibliography{refs}
\bibliographystyle{plain}

\newpage
\appendix
\section{Proof of Proposition \ref{prop:elementary-embedding-obstruction} (ii)}
\begin{proof}
Let \(r=\min\{p,nd\}\). Choose pairwise disjoint closed cubes \(C_1,\ldots,C_n\subset(0,1)^d\) and affine homeomorphisms \(T_j:K\to C_j\). Then
\[
    L=\{(T_1\mathbf t_1,\ldots,T_n\mathbf t_n):(\mathbf t_1,\ldots,\mathbf t_n)\in K^n\}
\]
is a labeled copy of \(K^n\). On \(L\), a permutation-invariant function can agree with an arbitrary ordered function, because membership in \(C_j\) identifies the \(j\)-th input.

Define \(H:K^n\to\mathbb R^p\) by projecting \(K^n=[0,1]^{nd}\) onto its first \(r\) coordinates and placing these coordinates in the first \(r\) coordinates of \(\mathbb R^p\). By the labeled-copy construction, extend \(H\) to a continuous permutation-invariant map \(F:K^n\to\mathbb R^p\) satisfying \(F(T_1\mathbf t_1,\ldots,T_n\mathbf t_n)=H(\mathbf t_1,\ldots,\mathbf t_n)\).

If \(F=\rho\circ\Psi\) exactly for continuous maps \(\Psi:K^n\to\mathbb R^M\) and \(\rho:\mathbb R^M\to\mathbb R^p\), then restrict to an \(r\)-dimensional coordinate face of \(L\) on which \(H\) is the identity on \([0,1]^r\). On this face, \(\rho\circ\Psi\) is injective, so \(\Psi\) must also be injective. Thus \([0,1]^r\) admits a continuous injective map into \(\mathbb R^M\). By invariance of domain, this is impossible when \(M<r\). Hence \(M\ge r=\min\{p,nd\}\).
\end{proof}
\section{Proof of Lemma~\ref{lem:janossy-finite-difference}}\label{app:lemma1proof}
\begin{proof} 
We prove the claim by induction on \(k\). For \(k=1\), an indexed \(1\)-Janossy latent map has the form
\[
    \Phi(x_1,\ldots,x_n)=\sum_{i=1}^n\phi_i(x_i).
\]
Thus
\[
    \Phi(a_1+y_1,x_2,\ldots,x_n)-\Phi(a_1,x_2,\ldots,x_n)
    =
    \phi_1(a_1+y_1)-\phi_1(a_1),
\]
which is independent of \(x_2,\ldots,x_n\).

Assume the result holds for indexed \((k-1)\)-Janossy latent maps, and let \(\Phi\) be an indexed \(k\)-Janossy latent map. Fix \(a_1,y_1\), and define
\[
    \Phi'(x_2,\ldots,x_n)
    =
    \Phi(a_1+y_1,x_2,\ldots,x_n)
    -
    \Phi(a_1,x_2,\ldots,x_n).
\]
We first check that \(\Phi'\) is an indexed \((k-1)\)-Janossy latent map in the variables \(x_2,\ldots,x_n\). Indeed, consider one summand of \(\Phi\),
\[
    \phi_{i_1,\ldots,i_k}(x_{i_1},\ldots,x_{i_k}).
\]
If none of the indices \(i_1,\ldots,i_k\) is equal to \(1\), then this summand is unchanged in the two terms defining \(\Phi'\), and hence cancels. Therefore only summands containing \(x_1\) survive.

For such a surviving summand, every occurrence of \(x_1\) is replaced by either \(a_1+y_1\) or \(a_1\). After this replacement, the summand depends only on the remaining variables \(x_2,\ldots,x_n\), and it can involve at most \(k-1\) of them, because at least one of the original \(k\) slots was occupied by \(x_1\). Thus each surviving term is a continuous function of at most \(k-1\) variables among \(x_2,\ldots,x_n\). If it depends on fewer than \(k-1\) variables, we regard it as a \((k-1)\)-ary summand by adding dummy arguments. If it depends on no remaining variables, it is simply a constant summand. Hence \(\Phi'\) is an indexed \((k-1)\)-Janossy latent map in the variables \(x_2,\ldots,x_n\).

Now apply the induction hypothesis to \(\Phi'\), using \(x_2,\ldots,x_k\) as the first \(k-1\) base variables. We obtain that
\[
    \sum_{\eta_2,\ldots,\eta_k\in\{0,1\}}
    (-1)^{\eta_2+\cdots+\eta_k}
    \Phi'(a_2+\eta_2y_2,\ldots,a_k+\eta_ky_k,x_{k+1},\ldots,x_n)
\]
is independent of the tail variables \(x_{k+1},\ldots,x_n\). Substituting the definition of \(\Phi'\), this expression becomes
\[
\begin{aligned}
&\sum_{\eta_2,\ldots,\eta_k\in\{0,1\}}
(-1)^{\eta_2+\cdots+\eta_k}
\Phi(a_1+y_1,a_2+\eta_2y_2,\ldots,a_k+\eta_ky_k,x_{k+1},\ldots,x_n)\\
&\quad-
\sum_{\eta_2,\ldots,\eta_k\in\{0,1\}}
(-1)^{\eta_2+\cdots+\eta_k}
\Phi(a_1,a_2+\eta_2y_2,\ldots,a_k+\eta_ky_k,x_{k+1},\ldots,x_n).
\end{aligned}
\]
This is exactly the \(k\)-fold alternating difference in \eqref{eq:janossy-finite-difference}, up to an overall sign. Since multiplying by an overall sign does not affect dependence on the tail variables, the \(k\)-fold alternating difference is independent of \(x_{k+1},\ldots,x_n\). This completes the induction step.
\end{proof}

\begin{figure}[t]
\centering
\begin{tikzpicture}[scale=0.92, every node/.style={font=\small}]

\node at (0,1.55) {tail directions for \(x_2\in[0,1]^2\)};

\fill (0,0) circle (2pt);
\node[below right] at (0,0) {\(c\)};
\draw (0,0) circle (0.9);

\draw[thick, red] (20:0.9) arc (20:140:0.9);
\draw[thick, blue] (140:0.9) arc (140:260:0.9);
\draw[thick, green!60!black] (260:0.9) arc (260:380:0.9);

\node[red] at (90:1.15) {\(B_1\)};
\node[blue] at (200:1.15) {\(B_2\)};
\node[green!60!black] at (320:1.15) {\(B_3\)};

\draw[->, red] (0,0) -- (90:0.68);
\draw[->, blue] (0,0) -- (210:0.68);
\draw[->, green!60!black] (0,0) -- (330:0.68);

\node[align=center] at (0,-1.55)
{\(x_2=c\pm\varepsilon u\)\\
\(u\in S^1\)};

\end{tikzpicture}
\caption{Case \(d=2,n=2,k=1\). The tail block has dimension \(d(n-k)=2\), so the perturbation directions form \(S^1\). The circle is covered by antipodal-free regions \(B_1,B_2,B_3\).}
\label{fig:case-d2-n2-k1}
\end{figure}
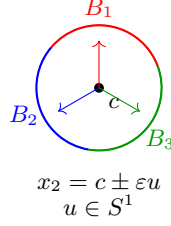

\section{Construction of the regions $B_r$}\label{appendix:regions}
The following is an explicit construction of the covering $B_r$ of $S^{b-1}$:

First, choose vertices \(\mathbf{v}_1,\ldots,\mathbf{v}_{b+1}\in\mathbb R^{b}\) of a regular simplex centered at the origin. Each region consists of the directions \(\mathbf{u}\) whose largest inner product is attained at a given simplex vertex:
\begin{equation}\label{eq:sphere-regions}
    B_r
    =
    \left\{
    \mathbf{u}\in S^{b-1}:
    \langle \mathbf{u},\mathbf{v}_r\rangle
    =
    \max_{1\le q\le b+1}\langle \mathbf{u},\mathbf{v}_q\rangle
    \right\},
    \qquad
    r=1,\ldots,b+1.
\end{equation}
The regions \(B_r\) cover \(S^{b-1}\), and for every \(r\) we have \(B_r\cap(-B_r)=\varnothing\). Indeed, if \(\mathbf{u}\in B_r\) and \(-\mathbf{u}\in B_r\), then for every \(q\),
\[
    \langle \mathbf{u},\mathbf{v}_r\rangle\ge \langle \mathbf{u},\mathbf{v}_q\rangle
    \qquad\text{and}\qquad
    \langle \mathbf{u},\mathbf{v}_r\rangle\le \langle \mathbf{u},\mathbf{v}_q\rangle .
\]
Hence \(\langle \mathbf{u},\mathbf{v}_r-\mathbf{v}_q\rangle=0\) for every \(q\neq r\). Since the vectors \(\mathbf{v}_r-\mathbf{v}_q\), \(q\neq r\), span \(\mathbb R^{b}\), this forces \(\mathbf{u}=0\), contradicting \(\mathbf{u}\in S^{b-1}\).

The construction in the proof corresponds to the case $b = d(n-k)$. Figure~\ref{fig:case-d2-n2-k1} illustrates the construction in the two-dimensional case \(b=2\), where the sphere is \(S^1\) and the simplex has three vertices.

\newpage

\end{document}